\PassOptionsToPackage{colorlinks=true,citecolor=ForestGreen,linkcolor=Plum,urlcolor=NavyBlue}{hyperref}
\documentclass[cleveref]{jmlr}

\usepackage{jmlrutils}

\usepackage{amsmath,amssymb,mathtools}
\usepackage{bbm}
\usepackage{physics}
\usepackage{framed}
\usepackage{graphicx}
\usepackage{multirow}
\usepackage{booktabs}
\usepackage{blkarray}
\usepackage{tcolorbox}
\usepackage[normalem]{ulem}
\usepackage[dvipsnames,table]{xcolor}

\usepackage{microtype}
\hypersetup{
  colorlinks=true,
  linkcolor=ForestGreen,
  citecolor=ForestGreen,
  urlcolor=ForestGreen,
}

\definecolor{easygreen}{RGB}{210,245,210}
\definecolor{hardred}{RGB}{255,220,220}
\definecolor{impossiblegray}{RGB}{235,235,235}

\usepackage{comment}
\usepackage{thmtools}
\usepackage{thm-restate}
\makeatletter
\def\thm@notefont{\normalfont} 
\def\thm@notebraces#1{ (#1)}   
\makeatother

\usepackage{physics}

\newcommand{\R}{\mathbb{R}}

\newcommand{\E}{\mathbb{E}}

\newcommand{\cP}{\mathcal{P}}

\renewcommand{\P}{\mathbb{P}}

\newcommand{\eps}{\varepsilon}

\newtheorem{assumption}{Assumption}

\renewcommand{\epsilon}{\varepsilon}

\numberwithin{equation}{section}

\newcommand{\argmin}{\mathop{\mathrm{argmin}}}

\newcommand{\EXP}{\mathbb{E}}

\newcommand{\defeq}{\stackrel{\mathrm{def.}}{=}}
\usepackage[normalem]{ulem}  
\usepackage{booktabs}

\DeclareMathOperator{\Var}{Var}

\newcommand{\themis}[1]{\textcolor{orange}{\textbf{Themis: } #1}}

\title[Mean Testing under Truncation]{Mean Testing under Truncation beyond Gaussian}
\usepackage{times}

\numberwithin{equation}{section}

\newcounter{mylineno}



\title{Mean Testing under Truncation beyond Gaussian}
\author{%
    \Name{Yuhao Wang} \Email{wanyuhao@amazon.co.jp}\\
    \addr Amazon  \\
    \Name{Roberto Imbuzeiro Oliveira } \Email{rimfo@impa.br}\\
    \addr IMPA, Rio de Janeiro, RJ, Brazil \\
    \Name{Themis Gouleakis} \Email{themis.gouleakis@ntu.edu.sg}\\
    \addr Nanyang Technological University
}

\begin{document}

\maketitle

\makeatletter
\let\@oddhead\@empty
\let\@evenhead\@empty
\thispagestyle{empty}
\pagestyle{plain}
\makeatother

\begin{abstract}
We characterize the fundamental limits of high-dimensional mean testing under arbitrary truncation, where samples are drawn from the conditional distribution $P(\cdot \mid S)$ for an unknown truncation set $S$ that may hide up to an $\eps$-fraction of the probability mass. 
For distributions with $p$-th directional moments of magnitude at most $\nu_{P,p}$,
truncation induces a bias of order $\mathcal{O}(\nu_{P,p}\eps^{1-1/p})$. This bias creates a sharp information-theoretic detectability floor: when the signal $\alpha$ falls below this threshold, the null and alternative hypotheses are indistinguishable even with infinite data.
Above this floor, we prove that a simple second-order test achieving near-optimal sample complexity
\[
n = \mathcal{O}\!\left(\frac{\|\Sigma_P\|}{(\alpha-4\nu_{P,p}\eps^{1-1/p})^2}\sqrt{d}\right).
\]
We further identify a structural escape from this finite-moment bias barrier.
Under a \emph{directional median regularity} assumption, truncation bias improves to linear order $\mathcal{O}(\eps)$.
This reveals an intermediate regime in which estimation requires $\Theta(d)$
samples for uniform recovery, while testing recovers the classical
$\Theta(\sqrt d)$ rate once truncation bias is eliminated.
Together, our results provide a unified framework for mean testing under truncation, connecting finite-moment, sub-Gaussian, and median-regular structural regimes.
\end{abstract}


\section{Introduction}
\label{sec:introduction}

Robust mean estimation has been extensively studied in statistics and machine learning, with a rich literature on contamination and heavy-tailed noise \citep{diakonikolas2016robust, lugosi2019mean,  diakonikolas2019recent, diakonikolas2023algorithmic,gupta2024beyond}. By contrast, the problem of \emph{mean testing} under arbitrary truncation, where samples are observed only if they fall in an unknown set, remains largely unexplored \citep{canonnegaussian}. 

Truncation arises naturally in settings such as privacy-preserving data collection, controlled sensing, and heavy-tailed environments \citep{dwork2014algorithmic, kamath2020private,zampetakis2025private}, where observations outside an unknown acceptance region are systematically discarded.
Whereas classical contamination models were primarily designed with outliers in mind \citep{huber2011robust, hampel1974influence}, truncation induces structured
missingness that can remove $\eps$-fraction of the probability mass anywhere in the
distribution. 
This difference has sharp consequences for hypothesis testing, which exhibits detectability thresholds that separate
irreducible truncation bias from statistical fluctuation.
This raises a fundamental question: \emph{to what extent can truncation obscure the mean, and when is reliable testing possible?}

We study the hypothesis testing problem
\begin{equation}
    H_0: \mu_P = 0 \quad \text{versus} \quad H_1: \|\mu_P\| \ge \alpha,
\end{equation}
given i.i.d.~samples from the truncated law $P(\cdot|S)$, where the truncation set $S \subseteq \mathbb{R}^d$ is unknown and may be chosen adversarially, subject only to $P(\cdot|S) \geq 1-\eps$. The main challenge is that truncation can induce a systematic bias in the observed mean, potentially masking the separation between null and alternative hypotheses. 
Quantifying this bias is therefore essential for characterizing the
detectability limits of the testing problem.

\begin{table}[t]
\centering
\caption{Phase diagram for high-dimensional mean testing under arbitrary truncation.
For each row, regimes are ordered by increasing signal strength $\alpha$.
Under finite $p$-moment assumptions, truncation induces a detectability floor
$b_{\mathrm{mom}}(\varepsilon)=\Theta(\nu_{P,p}\varepsilon^{1-1/p})$, below which
testing is information-theoretically impossible.
Under structural regularity, this floor improves to
$b_{\mathrm{reg}}(\varepsilon)=\tilde O(\varepsilon)$, revealing an intermediate
\emph{learning-hard} regime in which testing requires $\Theta(d)$ samples before
recovering the optimal $\Theta(\sqrt d)$ rate.
\emph{Gray} indicates impossibility, \emph{red} indicates $\Theta(d)$ sample
complexity, and \emph{green} indicates $\Theta(\sqrt d)$ complexity (all up to
logarithmic factors).}
\label{tab:comparison}

\renewcommand{\arraystretch}{1.15}
\small
\begin{tabular}{l
>{\centering\arraybackslash}p{0.16\columnwidth}
>{\centering\arraybackslash}p{0.17\columnwidth}
>{\centering\arraybackslash}p{0.16\columnwidth}
>{\centering\arraybackslash}p{0.16\columnwidth}}
\toprule
 & \parbox[c]{0.16\columnwidth}{\centering
   Impossible\\[-0.2ex]
   {\footnotesize $\alpha \le b_{\mathrm{reg}}(\varepsilon)$}
 }
 & \parbox[c]{0.18\columnwidth}{\centering
   Learning-hard\\[-0.2ex]
   {\footnotesize $b_{\mathrm{reg}}(\varepsilon) < \alpha \le b_{\mathrm{mom}}(\varepsilon)$}
 }
 & \parbox[c]{0.18\columnwidth}{\centering
   Near floor\\[-0.2ex]
   {\footnotesize $\alpha \gtrsim b_{\mathrm{mom}}(\varepsilon)$}
 }
 & \parbox[c]{0.18\columnwidth}{\centering
   Strong signal\\[-0.2ex]
   {\footnotesize $\alpha \gg b_{\mathrm{mom}}(\varepsilon)$}
 } \\
\midrule

\parbox[c]{0.20\columnwidth}{\raggedright
Moment-only\\[-0.2ex]
{\footnotesize (finite $p$-moments)}
}
& \cellcolor{impossiblegray}$\infty$
& \cellcolor{impossiblegray}$\infty$
& \cellcolor{easygreen}$\Theta(\sqrt d)$
& \cellcolor{easygreen}$\Theta(\sqrt d)$ \\

Structural regularity
& \cellcolor{impossiblegray}$\infty$
& \cellcolor{hardred}$\Theta(d)$
& \cellcolor{easygreen}$\Theta(\sqrt d)$
& \cellcolor{easygreen}$\Theta(\sqrt d)$ \\
\bottomrule
\end{tabular}
\end{table}

Our first contribution is a sharp characterization of truncation-induced bias under finite moment assumptions. For distributions with directional $p$-th moments upper bounded by $\nu_{P,p}$, we show in Lemma~\ref{lem:trunc-bias} that truncation can shift the mean by at most $\gamma=2\nu_{P,p}\eps^{1-1/p}$, and that this dependence is information-theoretically unavoidable. 
Building on this result, we design a simple test that achieves the optimal $\sqrt{d}$ dependence with sample complexity $ n = \mathcal{O}\!\left(\frac{\|\Sigma_P\|}{(\alpha - 2\gamma)^2}\sqrt{d}\right), $
whenever $\alpha > 2\gamma$. Conversely, we prove a matching impossibility result: when the signal lies below this bias floor, the null and alternative hypotheses become indistinguishable, even with infinitely many samples.

\paragraph{Informal Theorem (Finite-Moment Barrier).}
\textit{Consider testing $H_0:\mu_P=0$ versus $H_1:\|\mu_P\|\ge\alpha$ from
$\eps$-truncated samples of a distribution with finite directional
$p$-th moments upper bounded by $\nu_{P,p}$.
There exists a critical scale
$\alpha^\star =\Theta(\nu_{P,p}\eps^{1-1/p}$) such that:
(i) if $\alpha \ll \alpha^\star$, the hypotheses are information-theoretically
indistinguishable, regardless of sample size;
(ii) if $\alpha \gg \alpha^\star$, reliable testing is possible with optimal
$\sqrt{d}$ dependence in sample complexity (see Theorem \ref{thm:impossibility} in section \ref{subsec:impossibility}). }
\smallskip


Beyond this moment-limited regime, we identify a structural condition that allows one to escape the truncation bias barrier. Under an additional \emph{directional median regularity} assumption --- requiring each one-dimensional projection to be median-centered with non-vanishing local density --- the truncation bias improves to linear order $\mathcal{O}(\eps)$.  
This improvement is not a consequence of stronger tail concentration, but of
local identifiability of the center in every direction, which limits how much
probability mass can be removed near the center without altering the observed
distribution.
In particular, Gaussian distribution satisfy this condition, and the near-linear truncation bias established in \cite{canonnegaussian} arises as a special case of our analysis.
As a result, stable center estimation and reliable testing become possible even
for heavy-tailed distributions.

\paragraph{Informal Theorem (Structural Escape).}
\textit{If the distribution $P$ satisfies directional median regularity
(Assumption~\ref{assump:median-regularity}), then truncation induces at most a linear
$\mathcal{O}(\eps)$ bias in the center.
Consequently, reliable mean testing is possible using $\Theta(d)$ samples whenever
$\alpha \gtrsim \eps$, matching the bias scaling of the sub-Gaussian
regime even for heavy-tailed distributions (see Theorem \ref{thm:median_recovery} in section \ref{sec:median_regularity}).}
\smallskip

Together, these results provide a unified framework for mean testing under truncation, interpolating between the sub-Gaussian regime, the finite-moment impossibility barrier, and the median-regular setting where robustness is restored. As summarized in Table~\ref{tab:comparison}, the $\sqrt{d}$ dependence characteristic of high-dimensional testing is preserved across all feasible regimes, but the robustness (bias scale) varies fundamentally with the structural assumptions.

Our work positions truncation as a distinct and fundamental obstacle in high-dimensional testing, whose difficulty is governed by truncation-induced bias rather than statistical variance. 
By precisely characterizing detectability limits under finite-moment assumptions and identifying structural regularity conditions that eliminate this bias, we 
bridge classical robust statistics with modern distribution property testing, and provide principled testing procedures for truncated data.

\section{Related Work}
\label{sec:related}

The study of statistical inference under truncation has a long history, often
distinguishing between settings where the truncation mechanism is known and those
where it is unknown or adversarial. Our work bridges classical learning from
truncated data and modern high-dimensional hypothesis testing, focusing on
detectability limits rather than estimation accuracy.

The landscape of minimax signal detection and property testing is significantly shaped by several fundamental works that characterize the limits of distinguishing a signal from noise. \cite{baraud2002non} established sharp non-asymptotic minimax rates of testing for Gaussian sequence models over various constraint sets, including ellipsoids and $\ell_p$-bodies for $p \in [0, 2]$. Complementing this, \cite{ingster2012nonparametric} provides a comprehensive treatment of nonparametric goodness-of-fit testing under Gaussian models, defining the minimax separation rates that separate detectable signals from indistinguishable noise in the asymptotic regime. While these classical results define the information-theoretic limits for testing under standard Gaussian noise, they generally assume the observations follow a known distribution without missingness or adversarial intervention. In contrast, the truncation model studied here introduces a systematic bias barrier of order $\mathcal{O}(\nu_{P,p}\epsilon^{1-1/p})$ that persists even with infinite data, requiring structural regularity to overcome.

\paragraph{Mean estimation and learning from truncated samples.}
Classical work on truncated and censored samples focuses on parametric estimation
(e.g., for exponential or normal families), typically assuming a known truncation
mechanism and developing method-of-moments or maximum-likelihood procedures
\citep{deemer1955estimation,cohen1957solution,dixon1960simplified,fisher1931properties,cohen1991truncated}.
More recently, algorithmic work studied high-dimensional learning tasks from
truncated data under known truncation sets, including efficient algorithms for
truncated statistics and truncated regression
\citep{daskalakis2018efficient,daskalakis2019computationally,daskalakis2020truncated},
as well as structured truncation arising from generative models such as one-layer
ReLU networks \citep{wu2019learning}. Beyond known-truncation settings, efficient
algorithms for unknown truncation have also been developed for certain parametric
families \citep{kontonis2019efficient,DBLP:conf/focs/LeeMZ24}, and for regression /
classification under truncation mechanisms motivated by self-selection
\citep{ilyas2020theoretical,cherapanamjeri2023makes}.

\paragraph{Mean testing under truncation and detecting truncation.}
Most of the above literature focuses on learning/estimation, whereas our focus is
hypothesis testing. Closest in spirit, \cite{canonnegaussian} studies Gaussian mean
testing under \emph{arbitrary} truncation, identifying sharp limits and a
near-linear $\mathcal{O}(\eps\sqrt{\log(1/\eps)})$ truncation bias barrier in the
Gaussian/sub-Gaussian regime. Our work generalizes the testing perspective beyond
Gaussian assumptions and shows that for general distributions with finite moments,
truncation induces a polynomial bias barrier of order $\mathcal{O}(\eps^{1-1/p})$,
leading to an information-theoretic detectability floor.
Orthogonally, \cite{DeNS23,DBLP:conf/stoc/De0NS24} consider problems of
\emph{detecting} whether samples have been truncated under structural promises on
the truncation set or mechanism; these are complementary to our goal, which assumes
truncation may be present and seeks to test the underlying mean.

\paragraph{Robustness: truncation vs.\ contamination.}
Robust statistics traditionally studies contamination models (e.g., $\eps$-contamination),
where an adversary replaces a fraction of samples with arbitrary outliers
\citep{huber2011robust,diakonikolas2023algorithmic,depersin2022robust,diakonikolas2020outlier}. While truncation can be viewed
as a restricted form of adaptive (strong) contamination---since an adversary can
remove informative mass---it is fundamentally different: truncation creates
systematic missingness and can ``hollow out'' the distribution near its center.
For robust mean testing under contamination, \cite{canonne2023full} characterizes
sharp separations between oblivious and adaptive contamination, including regimes
where testing is strictly easier than learning. In contrast, under truncation we
show that testability is governed by truncation-induced bias floors and, under
structural regularity, an intermediate learning-hard regime appears before the
optimal $\Theta(\sqrt{d})$ testing rate is recovered. 
Other recent works have established powerful methods for robust mean estimation—such as the nearly linear-time algorithms by Depersin and Lecué \cite{depersin2022robust}  and the stability-based framework by Diakonikolas et al. \cite{diakonikolas2020outlier}  —in the aforementioned adversarial contamination model. 
\paragraph{Heavy-tailed mean estimation and testing.}
There is a rich literature on robust mean estimation under finite-moment or finite-covariance assumptions,
including algorithmic approaches achieving sub-Gaussian-type performance and sharp uniform guarantees
\citep{hopkins2020robust,gupta2024beyond,minsker2025uniform,hogsgaard2025uniform}.
Recent work also highlights qualitative separations between heavy-tailed estimation and adversarial contamination
\citep{cherapanamjeri2025heavy}. Our results identify a testing-specific ``moment barrier'' under truncation:
even when accurate estimation is possible (in the absence of truncation) under heavy tails, truncation induces a systematic bias of order
$\eps^{1-1/p}$ that renders mean testing information-theoretically impossible below the corresponding
detectability floor unless additional structural regularity is imposed.

\paragraph{Geometric and median-based regularity.} 
Median-based stability and geometric regularity conditions are widely used in robust statistics to overcome
moment limitations (see, e.g., surveys and references in \cite{lugosi2019mean,diakonikolas2023algorithmic}).
Our \emph{directional median regularity} assumption is motivated by this line of work; 
it imposes a local density near the center in all directions, thereby ruling out adversarial removal of probability mass near the center.
We show that such structural regularity improves the truncation-induced bias from $\mathcal{O}(\eps^{1 - 1/p})$ to $\mathcal{O}(\eps)$ and reveals an intermediate regime where learning remains hard while testing becomes tractable. 

\section{Notations and Preliminaries}
\label{sec:prelims}

This section introduces the mathematical framework and distributional quantities required to analyze mean testing under truncation.

\paragraph{Notations}
We consider families of probability distributions $\cP$ over $\R^d$. For any $P \in \cP$, $\mu_P := \E_{X \sim P}[X]$ denotes the mean vector and $\Sigma_P := \mathrm{Cov}_{X \sim P}(X)$ denotes the covariance matrix. We denote both the Euclidean norm of vectors and the spectral (operator) norm of a matrix by $\|\cdot\|$. Samples $X_1, \dots, X_n$ are assumed to be independently and identically distributed (i.i.d.). For a real parameter $p$ and threshold $\kappa$, we write $p \uparrow \kappa$
to denote the limit as $p$ increases to $\kappa$ from below.
Accordingly, statements of the form ``$\nu_{P,p}$ diverges as $p \uparrow \kappa$''
mean that $\nu_{P,p}<\infty$ for all $p<\kappa$ and $\nu_{P,p}=\infty$ for $p\ge \kappa$.

\paragraph{The Truncation Model}
Let $\eps \in [0, 1/2]$ for all $P\in \cP$. We say that samples are subject to an $\eps$-truncation if there exits a measurable set $S \in \R^d$ satisfying $P(S) \ge 1 - \eps$. 
We define the \emph{truncated law} $P_S := P(\cdot \mid S)$, representing the conditional distribution of $P$ restricted to $S$.
Throughout the paper, we observe i.i.d.~samples $X_1, \dots, X_n \sim P_S$, while the goal is to test
\[
H_0: \mu_P = 0
\qquad \text{versus} \qquad
H_1: \|\mu_P\| \ge \alpha,
\]
for some separation parameter $\alpha > 0$. 

\paragraph{Directional Moments and Tail Geometry}
Since truncation may remove mass along arbitrary direction in $\R^d$, we characterize tail behavior via worst-case one-dimensional projection.

\begin{definition}[Directional $p$-th Moment]
\label{def:pthmoment}
For $p \ge 1$, the directional $p$-th moment bound of a distribution $P$ on $\mathbb{R}^d$ is defined as:
\begin{equation}
    \nu_{P,p} := \sup_{v \in \mathbb{R}^d : \|v\|=1} \left( \E_{X \sim P} |\langle v, X - \mu_P \rangle|^p \right)^{1/p},
\end{equation}
that is, the worst-case $p$-th moment over all one-dimensional projections.
\end{definition}

The growth of $\nu_{P,p}$ as a function of $p$ characterizes the tail geometry of $P$. 
For sub-Gaussian distributions with parameter $\sigma$, there exists a constant $C>0$ such that $\nu_{P,p} \le C\sigma\sqrt{p}$ for all $p \ge 1$. 
In contrast, for heavy-tailed distributions with polynomially decaying tails, directional moments typically exist only up to some finite order $\kappa$: the quantity $\nu_{P,p}$ remains bounded for $p<\kappa$ but diverges as $p \uparrow \kappa$. 

Appendix~\ref{appendix:section3} discusses the behavior of $\nu_{P,p}$ for sub-Gaussian and heavy-tailed distributions.  
\paragraph{Fundamental Limits of Truncation Bias}
The following lemma establishes the technical foundation for our testing results by quantifying how much an unknown truncation set $S$ can shift the population parameters. 

\begin{restatable}[Truncated vs.\ Non-truncated Parameters]{lemma}{TruncationBiasLemma}\label{lem:trunc-bias}
Let $P$ be a distribution on $\R^d$ and $S$ be a measurable set with $P(S) \ge 1 - \eps$ for $\eps \in [0, 1/2]$. Then for any $p \ge 1$:
\begin{enumerate}
    \item \textbf{Mean Shift:} $\|\mu_P - \mu_{P_S}\| \le 2\nu_{P,p}\eps^{1-1/p}$.
    \item \textbf{Covariance Bound:} $\|\Sigma_{P_S}\| \le (1-\eps)^{-2}\|\Sigma_P\| \le 4\|\Sigma_P\|$.
\end{enumerate}
\end{restatable}

The $\mathcal{O}(\eps^{1-1/p})$ shift in the mean  represents the systematic bias that our test must overcome. 
Moreover, if $P$ is sub-Gaussian so that $\nu_{P,p} \lesssim \sigma\sqrt{p}$ for all $p$, optimizing the bound over $p$ yields the nearly linear truncation bias
$\mathcal{O}(\eps\sqrt{\log(1/\eps)})$.

\section{Mean Testing under Finite-Moment Truncation}
\label{sec:moment_testing}

This section investigates the fundamental limits of high-dimensional mean testing when samples are subject to arbitrary truncation and the underlying distribution $P$ possesses only finite directional moments. Our analysis
demonstrates that truncation induces a systematic detectability floor that is information-theoretically unavoidable under moment-only constraints.

\subsection{Effective Signal and the Detectability Floor}
\label{subsec:effective_signal}

The central challenge in mean testing under truncation is that, the learner observes samples from the conditional law $P_S$ rather than the target distribution $P$. As shown in Lemma~\ref{lem:trunc-bias}, for distributions with finite directional $p$-th moments, a worst-case truncation set $S$ can induce a systematic shift of the population mean of order $\gamma:=2\nu_{P,p}\varepsilon^{1-1/p}$.
This motivates the definition of the \emph{effective signal strength}
\begin{equation}
\label{eq:effective_signal}
    \beta := \max\left(0, \alpha - \gamma\right).
\end{equation}
which quantifies the residual separation between the alternative hypothesis under truncation and the mean $0$ case (i.e null hypothesis without truncation). 
Under the alternative hypothesis $H_1$ for distribution $P$, the observed mean of distribution $P_S$ satisfies $\|\mu_{P_S}\| \ge \beta$, whereas under the null hypothesis $H_0$, truncation alone may induce an observed mean of magnitude $\|\mu_{P_S}\|\le \gamma$. Consequently, testing performance is governed by whether the effective signal $\beta$ is above this threshold (i.e $\beta > \gamma$).

This identifies a truncation-induced detectability floor at scale $b(\eps) = \Theta(\gamma)$. Indeed, by Lemma~\ref{lem:trunc-bias}, truncation can shift the population mean by at most $b(\eps)$ in an arbitrary direction. Under the null hypothesis $H_0$, this implies $\| \mu_{P_S} \| \le b(\eps)$, while under the alternative $H_1$ we have the deterministic lower bound $\| \mu_{P_S} \| \ge \|\mu_P \| - b(\eps) \ge \beta$ by triangle inequality. 
When $\alpha \le c\,b(\eps)$ for some universal constant $c >0$, the null and alternative become information-theoretically indistinguishable even with infinite samples (Theorem~\ref{thm:impossibility}), whereas for sufficiently large effective signal $\beta$ reliable testing is possible.

By centering the analysis around the parameter $\beta-\gamma=\alpha-2\gamma$, which is required to be positive, we separate systematic bias due to truncation from the statistical fluctuations that determine variance and finite-sample behavior in the subsequent sections.

\subsection{A Natural Second-Order Test Statistic}
\label{sec:variance}
To distinguish between the null hypothesis $H_0: \|\mu_P\| = 0$ and alternative hypothesis $H_1: \|\mu_P\| \ge \alpha$, we consider the order-2 U-statistic:
\begin{equation}
\label{eq:ustat}
    F := \binom{n}{2}^{-1} \sum_{1 \le i < j \le n} \langle X_i, X_j \rangle,
\end{equation}
where $X_1, \dots, X_n \sim P_S$ are i.i.d.

This statistic admits a simple interpretation. For any distribution  $Q$ on $\mathbb{R}^d$ with finite second moments,
\begin{equation}
\label{eq:ustat_expectation}
    \mathbb{E}_{X,X' \sim Q}\!\left[\langle X, X' \rangle\right] = \|\mu_Q\|^2.
\end{equation}
Consequently, $F$ is an unbiased estimator of $\|\mu_Q\|^2$ when $X_i \sim Q$.
In contrast to the plug-in estimator $\bigl\|\frac{1}{n}\sum_{i=1}^n X_i\bigr\|^2$, whose expectation incurs an $\mathcal{O}(d/n)$ bias that dominates the signal in high dimensions, the $U$-statistic $F$ avoids this bias and isolates the second-order signal $\|\mu_Q\|^2$. This property is essential for achieving meaningful testing guarantees when $d$ is large.

Under truncation, the observed samples follow $Q = P_S$, and hence $\E[F] = \|\mu_{P_S}\|^2$. In view of the effective signal definition \eqref{eq:effective_signal}, this implies
\[
\E[F] \ge \beta^2 \quad \text{under } H_1,
\qquad
\mathbb{E}[F] \le \gamma^2 \quad \text{under } H_0.
\]
Therefore, a necessary condition for the hypothesis testing to succeed is the effective signal strength $\beta$ to be at least $\gamma$ 
and thus
$\alpha > 2\gamma$.

We emphasize that the specific choice of the statistic \eqref{eq:ustat} is not crucial: any second-order statistic estimating $\|\mu_{P_S}\|^2$ would lead to the same detectability behavior. The fundamental difficulty of the problem arises from truncation-induced bias rather than the choice of estimator; statistical fluctuations and confidence guarantees are analyzed in the subsequent sections.

\subsection{Variance Control under Finite Moments}
\label{subsec:variance_control}

We now analyze the variance of the test statistic $F$ defined in \eqref{eq:ustat} under the truncated law $P_S$. In contrast to the systematic bias induced by truncation, we show that truncation does not fundamentally worsen the variance of second-order statistics, provided the underlying distribution has finite second moments. The following lemma bounds the variance of the order-2 U-statistic in \eqref{eq:ustat}.

\begin{restatable}[Variance Bound for U-statistic]{lemma}{VarianceBoundLemma}
\label{lem:variance_bound}
Fix a distribution $Q$ and let $X_1,\dots,X_n \sim Q$ be i.i.d. Then the variance of the $U$-statistic
\[
F := \binom{n}{2}^{-1}\sum_{1\le i<j\le n}\langle X_i, X_j\rangle
\]
satisfies
\begin{equation}
\label{eq:variance_bound}
\Var_{Q}(F)
\;\le\;
\frac{4}{n}\,\langle \mu_{Q}, \Sigma_{Q} \mu_{Q} \rangle
\;+\;
\frac{4}{n(n-1)}\,\tr(\Sigma_{Q}^2).
\end{equation}
\end{restatable}

The proof follows from a standard Efron--Stein inequality for $U$-statistics. Importantly, the bound \eqref{eq:variance_bound} depends only on second-order moments of the truncated distribution.

We apply Lemma \ref{lem:variance_bound} for $Q=P_S$, which denotes sampling from the truncated distribution $P$ with truncation set $S$ and relate this variance bound to the parameters of the original distribution $P$ as follows:
\begin{equation}
\label{eq:variance_bound_P}
\Var_{P_S}(F)
\;\lesssim\;
\frac{\|\Sigma_P\|\,\|\mu_{P_S}\|^2}{n}
\;+\;
\frac{d\,\|\Sigma_{P_S}\|^2}{n^2}\;\lesssim\;
\frac{\|\Sigma_P\|\,\|\mu_{P_S}\|^2}{n}
\;+\;
\frac{d\,\|\Sigma_{P}\|^2}{n^2},
\end{equation}
where the last inequality follows by invoking Lemma~\ref{lem:trunc-bias}, which yields $\|\Sigma_{P_S}\| \le 4\|\Sigma_P\|$, as well as the bound $\tr(\Sigma_{P_S}^2)\leq d\,\|\Sigma_{P_S}\|^2$. 

This variance bound highlights a key contrast between bias and variance under truncation. While truncation can induce a mean shift of order $\eps^{1-1/p}$, it does not affect the variance upper bound.
The requirement $p \ge 2$ is essential here, as it guarantees the existence of second moments needed to control \eqref{eq:variance_bound}; when $p < 2$, even the variance of $F$ may be ill-defined.

A detailed proof of Lemma~\ref{lem:variance_bound} is deferred to Appendix~\ref{app:variance}.

\subsection{Achievable Testing Rates}
\label{subsec:achievable_rates}

The variance bound for the $U$-statistic $F$ implies that a simple test achieves a constant probability of success, once the sample size is large enough.

\begin{restatable}[Mean Testing under truncation - constant error prob.]{theorem}
{MeanTestTruncConstant}
\label{thm:mean-test-trunc-const}
Let $S\subset\mathbb{R}^d$ satisfy $P(S)\ge 1-\epsilon$ with $\epsilon\in[0,1/2]$, and let
$X_1,\dots,X_n \stackrel{\text{iid}}{\sim} P_S := P(\cdot\,|\,S)$.
Let $\gamma=2\,\nu_{P,p}\,\epsilon^{\,1-1/p}$ and  $\alpha>2\gamma$. 
Consider $F$ as above and the test $T=\mathbf{1}\!\left\{\,F>\frac{\alpha^2}{4}\right\}$. There exists a universal constant $C>0$ such that, if
 \[
n \;\ge\;\max\left\{2, C\,\frac{\|\Sigma_{P}\|\sqrt{d}}{(\alpha-2\gamma)^2}\right\},
\]
then $T$ separates $\|\mu_P\|=0$ from $\|\mu_P\|\ge \alpha$ with probability at least $2/3$.

\end{restatable}

The proof of this result uses Chebyshev's inequality to bound the fluctuations of $F$ around $\|\mu_{P_S}\|^2$. Then the standard deviation and the norm of  $\mu_{P_S}$ are compared under both $H_0$ and $H_1$ to obtain the result.

From Theorem \ref{thm:mean-test-trunc-const}, one can easily boost the separation probability to any value $1-\delta$ by performing the test $T$ several times over independent samples. This is the content of the next result. 

\begin{restatable}[Mean Testing under Truncation]{theorem}{TestingUnderTruncation}
\label{thm:achievable_rates}In the same setting as Theorem \ref{thm:mean-test-trunc-const}, given $\delta\in (0,1/3)$,
there exists a test $T$ that distinguishes $H_0: \mu_P = 0$ from $H_1: \|\mu_P\| \ge \alpha$ with probability at least $1-\delta$ whenever $\alpha > 2\gamma$ 
and the sample size satisfies
\begin{equation}
\label{eq:sample_complexity}
    n \;\ge\; C\,\left\{2,\frac{\|\Sigma_P\|}{(\alpha-2\gamma)^2}\,\sqrt{d}\,\right\} \log(1/\delta),
\end{equation}
where $C>0$ is a universal constant.
\end{restatable}

The proofs of these results are provided in Appendix~\ref{app:variance}.

\paragraph{Discussion of the Rate.}
The sample complexity in \eqref{eq:sample_complexity} exhibits several notable features. 
First, the $\sqrt{d}$ dependence matches the optimal rate for high-dimensional mean testing, demonstrating that arbitrary truncation does not fundamentally alter the dimensional scaling of the problem. 
Second, the testing procedure is functionally parameter-agnostic: it does not require explicit knowledge of the moment order $p$ or the bound $\nu_{P,p}$, which only influence feasibility through the parameter $\gamma$. 
Finally, in the limit $p\to\infty$, the truncation bias recovers the near-linear $O(\varepsilon\sqrt{\log(1/\varepsilon)})$ behavior characteristic of sub-Gaussian distributions, aligning our result with known optimal rates in the Gaussian setting.

\subsection{Information-Theoretic Impossibility}
\label{subsec:impossibility}
We now show that the effective signal threshold identified in Section~\ref{subsec:effective_signal} is unavoidable under finite-moment assumptions. In particular, when the truncation-induced bias dominates the signal, no test can reliably distinguish the null and alternative hypotheses, regardless of sample size.

\begin{theorem}[Impossibility under severe truncation]
\label{thm:impossibility}
Fix $p\ge2$ such that $\nu_{P,p}<\infty$.
There exists a universal constant $C>0$ such that if $\alpha \le C\,\nu_{P,p}\eps^{\,1-1/p}$, 
then there exists a distribution $P$ on $\R^d$ with mean vector $\mu_P$ satisfying
$ \|\mu_P\|\ge \alpha$, a distribution $P_0$ with $\mu_{P_0}=0$, 
and a measurable set $S\subseteq \R^d$ with $P(S)\ge 1-\eps$ such that
\[
P(\cdot\mid S)=P_0.
\]
Consequently, no (possibly randomized) test based on i.i.d.\ samples from the truncated
distribution $P(\cdot\mid S)$ can distinguish the hypotheses
$\|\mu_P\|=0$ versus $\|\mu_P\|\ge\alpha$ with error probability bounded away from $1/2$,
even as $n\to\infty$.
\end{theorem}

\begin{proof} Fix $p\geq 2$.
Let $\delta_x$ denote the Dirac measure at $x\in\mathbb{R}$ and $P = (1 - \varepsilon)\delta_0 + \varepsilon\delta_{ \varepsilon^{-1/p}}$ be a distribution supported on $\{\varepsilon^{-1/p},0\}$. We have that $\mu_P=\varepsilon^{1-1/p}$ and $\nu_{P,p}\leq 2$. Letting $P_0=\delta_0$ 
and $S=(-\infty,0]$ completes the proof, since $P(\cdot|S)=P_0=\delta_0$ and $P(S)=1-\varepsilon$. \end{proof}

\noindent This result establishes the sharp boundary between the achievable and impossible regimes:
when the hidden mass carries as much mean energy as the target signal, truncation obscures it.

\subsection{Discussion: The Moment Barrier}
\label{subsec:moment_barrier}

Sections~\ref{subsec:effective_signal}–\ref{subsec:impossibility} 
establish a sharp detectability threshold for mean testing under truncation.
Under finite-moment assumptions alone, truncation induces an unavoidable bias of order $\mathcal{O}(\eps^{1-1/p})$, yielding an information-theoretic barrier that persists even with
infinitely many samples.

This phenomenon arises from the local asymmetry permitted by moment constraints.
An adversary may remove an $\eps$-fraction of mass in a highly asymmetric manner
near the center of the distribution, thereby shifting the mean while leaving the observed truncated distribution
unchanged. Because the mean is the minimizer of a global $L_2$ objective, it is inherently
sensitive to such localized mass removal. As a result, truncation can selectively eliminate
the portion of the distribution carrying the signal, rendering the null and alternative
hypotheses indistinguishable.
Overcoming this barrier requires structural assumptions that constrain how probability mass
is arranged locally near the center, preventing truncation from completely obscuring the signal.

\section{Breaking the Moment Barrier via Median Regularity}
\label{sec:median_regularity}

Section~\ref{sec:moment_testing} showed that under moment-only assumptions,
truncation induces an unavoidable detectability floor:
an $\eps$-fraction removal of mass can shift the mean by
$\Theta(\eps^{1-1/p})$, rendering inference impossible below this scale.
This barrier arises because moment assumptions impose only global tail control
and do not constrain the local geometry of the distribution near its center.

In this section, we show that this limitation is not fundamental.
By imposing a mild \emph{local} structural condition, \emph{directional median regularity},
we rule out the local hollowing phenomenon responsible for the moment barrier.
As a result, truncation induces only linear $\mathcal{O}(\eps)$ bias at the
population level.
However, this improvement comes at a cost: recovering the center uniformly over
all directions requires sample complexity linear in the ambient dimension.

\subsection{Directional Median Regularity}
To bypass the moment-induced barrier identified in
Section~\ref{sec:moment_testing}, we introduce a local structural assumption on
the distribution near its center.
This assumption is designed to capture the minimal regularity needed for robust
inference under truncation, and does not rely on finite variance or higher-order
moment bounds.
\begin{assumption}[Directional Median Regularity]
\label{assump:median-regularity}
Let $P$ be a distribution on $\mathbb{R}^d$ with mean $\mu_P$.
We assume that for every unit vector $v \in \mathbb{S}^{d-1}$:
\begin{itemize}
    \item (Directional centering).
    $ 
    \mathrm{median}\!\left(\langle X - \mu_P, v\rangle\right) = 0 .
    $
    \item (Local mass).
    The projection $Z_v := \langle X - \mu_P, v\rangle$ admits a density $f_v$
    satisfying
    \[
    f_v(t) \ge c
    \qquad \text{for all } |t| \le r ,
    \]
    where $c,r>0$ do not depend on $v$.
\end{itemize}
\end{assumption}

Assumption~\ref{assump:median-regularity} enforces a uniform lower bound on the
one-dimensional densities of all projections in a neighborhood of the center.
As a consequence, for each direction $v$, small perturbations of the projected
distribution function translate into proportionally small shifts of the
directional median.
This property rules out distributions for which an $\eps$-fraction of mass
can be removed near the center without noticeably affecting the observed law,
which is precisely the mechanism underlying the moment-induced bias barrier.
Under this assumption, truncation-induced bias improves from polynomial to
linear order.
The following lemma formalizes this stability at the population level. 

\begin{restatable}[Directional Median Stability]{lemma}{MedianStability}
\label{lem:median_stability} 
Suppose that the pair $(P,S)$ is such that $P$ satisfies Assumption~\ref{assump:median-regularity} and 
$P(S)\ge 1-\varepsilon$ for some $\varepsilon<cr$.
Then for every $v\in\mathbb{S}^{d-1}$,
\[
\bigl|
\mathrm{median}(\langle X,v\rangle \mid X\in S)
- \langle \mu_P,v\rangle
\bigr|
\;\le\;
\frac{\varepsilon}{c}.
\]

\end{restatable}

It is instructive to consider the above lemma when $d=1$. In this case, Assumption \ref{assump:median-regularity} applied to the density $f$ of $X-\mu_P$ gives:
\[2cr\leq \int_{-r}^rf(t)\,dt\leq 1,\]
so that $cr\leq 1/2$, with equality if and only if $f$ is uniform over an interval of length $2r$ (and then $c=1/2r$). The next result shows that, if we are not in this degenerate case --- that is, if $cr<1/2$ --- then no analogue of Lemma \ref{lem:median_stability} can hold with $\varepsilon>2cr$. 

\begin{restatable}[Sharpness of median stability condition]{lemma}{MedianInstability}
\label{lem:median_instability} Given $0<\xi<1/2$, $\varepsilon>2\xi$,   and $R>0$, there exists a distribution $P$ over $\R$ and a set $S\subset \R$ with the following properties.
\begin{enumerate}
\item $P$ satisfies Assumption \ref{assump:median-regularity} with constants $c,r>0$ such that $cr=\xi$;
\item $P(S)\geq 1-\varepsilon$;
\item \(
\bigl|
\mathrm{median}(\langle X,v\rangle \mid X\in S)
- \langle \mu_P,v\rangle
\bigr|
\;\ge\;
R.
\)

\end{enumerate}
Moreover, the truncated distribution $P_S$ satisfies the second part of Assumption \ref{assump:median-regularity} for certain constants $c_0,r_0>0$. 
\end{restatable}

\subsection{Population Stability under Truncation}
\label{subsec:population_stability}

Lemma~\ref{lem:median_stability} shows that under directional median regularity,
truncation induces at most a linear $\mathcal{O}(\varepsilon)$ displacement of the
center in every direction.
This behavior is fundamentally different from the moment-limited regime studied
in Section~\ref{sec:moment_testing}, where truncation can induce a bias of order
$\varepsilon^{1-1/p}$.

The underlying mechanism is purely one-dimensional.
Conditioning on a truncation set $S$ with $P(S)\ge 1-\varepsilon$ perturbs the
distribution function of each projection $\langle X,v\rangle$ by at most
$\varepsilon$ uniformly.
When combined with a uniform lower bound on the density near the median, this
perturbation translates directly into a bound on the shift of the directional
median.
In this sense, the directional median acts as a \emph{stable pivot} whose location
is controlled by the total mass removed, rather than by the behavior of the tails.

This stability property rules out the failure mode responsible for the
moment-induced bias barrier.
Under moment-only assumptions, the distribution may place arbitrarily little mass
near its center, allowing truncation to remove signal while leaving the observed
truncated distribution essentially unchanged.
Directional median regularity excludes this possibility by enforcing local
identifiability of the center in every direction.

\subsection{Robust Estimation and Dimensional Sample Complexity}

We now turn to estimation from finite samples drawn from the truncated law $P_S$.
Lemma~\ref{lem:median_stability} implies that the population center remains
$\mathcal{O}(\varepsilon)$-close to $\mu_P$ in every direction.
To exploit this stability algorithmically, it suffices to estimate directional
medians uniformly over all directions.

Given i.i.d.\ samples $X_1,\ldots,X_n\sim P_S$, let $\hat m_v$ denote the empirical
median of $\langle X_1,v\rangle,\ldots,\langle X_n,v\rangle$.
We define
\[
\hat\mu
\;:=\;
\arg\min_{u\in\mathbb{R}^d}
\sup_{v\in\mathbb{S}^{d-1}}
\bigl|\langle u,v\rangle-\hat m_v\bigr|.
\]

\begin{restatable}[Robust Recovery under Median Regularity]{theorem}{MedianRecovery}
\label{thm:median_recovery} There exists a universal $C_0>0$ such that the following holds. Let $P$ satisfy Assumption~\ref{assump:median-regularity} with a constant $c>0$ and let
$S$ be any truncation set with $P(S)\ge1-\varepsilon$, where $0\leq \varepsilon< cr$.
Given $n$ i.i.d.\ samples from $P_S$, where $n$ satisfies
\[n>\frac{C_0}{(cr-\varepsilon)^2}\,(d+\log(1/\delta))\]

the estimator $\hat\mu$ satisfies, with
probability at least $1-\delta$,
\[
\|\hat\mu-\mu_P\|
\;\le\;
\frac{2}{c}\!\left(
\varepsilon
+
C\sqrt{\frac{d+\log(1/\delta)}{n}}
\right) .
\]
\end{restatable}

Theorem~\ref{thm:median_recovery} shows that directional median regularity restores
near-optimal robustness under truncation.
However, achieving $\mathcal{O}(\varepsilon)$ accuracy requires
$n=\Theta(d)$ samples.
This dependence is unavoidable and reflects the VC dimension of the class of
directional halfspaces, which must be controlled to ensure uniform stability
across all projections. 

\begin{remark}We also note that some restriction relating $\varepsilon$ to $<cr$  is necessary for our estimator to work. By Lemma  \ref{lem:median_instability}, $\varepsilon>2cr$ allows from truncated distributions whose medians are arbitrarily far from the un-truncated median. Hovever, the ``bad truncated distribution''~$P_S$ produced by Lemma \ref{lem:median_instability} does not satisfy Assumption \ref{assump:median-regularity} because its median and mean are not equal. \end{remark}

\subsection{Testing versus Estimation under Structural Regularity}

Directional median regularity eliminates truncation-induced bias for estimation,
but it does not remove the intrinsic dimensional difficulty of hypothesis testing.
Even in the absence of truncation, testing a $d$-dimensional mean is known to
require $\mathcal{O}(\sqrt{d}/\beta^2)$ samples, where $\beta$ denotes the effective signal
strength.

\begin{restatable}[Testing under Structural Regularity]{theorem}{TestingRegularity}
\label{thm:testing_regularity}
Fix $d\ge1$ and $\delta\in(0,1/3)$.
Consider the problem of testing $H_0:\mu_P=0$ versus $H_1:\|\mu_P\|\ge \alpha>0$
from $n$ i.i.d.\ samples drawn from a truncated distribution $P_S$ over $\mathbb{R}^d$, where $P(S)\geq 1-\varepsilon$. Then there exist constants $c,C>0$ such that:
\begin{itemize}
\item If $P$ satisfies Assumption \ref{assump:median-regularity} with constants $c,r>0$, $\varepsilon<\min\{cr,\alpha/8c\}$ and \[n\geq \frac{C}{\alpha^2}\left(d+\log(1/\delta)\right),\] there exists a test with probability of error at most $\delta$, uniformly over $H_0$ and $H_1$.
\item If $P$ is a normal distribution over $\R^d$ with unknown mean $\mu_P$ and known covariance $\Sigma_P=I_{d}$ (which in particular satisfies Assumption \ref{assump:median-regularity} for $c=\sqrt{e/2\pi}$, $r=1$), then for any $\varepsilon\in (0,1)$ and  $n\le c\,d/\varepsilon$, there exists a choice of $\alpha=\Theta(\varepsilon\sqrt{\log(1/\varepsilon)})$ such that no test can achieve error smaller than $1/3$ uniformly over both $H_0$ and $H_1$.
\end{itemize}
\end{restatable}
The theorem shows that, for $\varepsilon\lesssim \alpha$ and suitable parameters $c,r$ for $P$, the sample complexity of testing $H_0$ versus $H_1$ over directional median regular distributions is $\Theta_\alpha(d)$. This generalizes the results for \cite{canonnegaussian} for Gaussian distributions; in fact, the second part of the theorem follows from Lemma 3.5 in that paper. Theorem \ref{thm:median_recovery} highlights a fundamental separation:
directional mean regularity removes the truncation bias barrier $\varepsilon^{1-1/p}\lesssim\alpha$ for learning, but does
not circumvent the inherent dimensional cost $\Theta(d)$ of high-dimensional testing.

\section{Conclusion and Open Directions}
\label{sec:conclusion}

We studied high-dimensional mean testing under arbitrary truncation and
characterized the fundamental limits imposed by tail behavior and structural
regularity. Our results identify truncation-induced bias as the primary obstacle
to testing: under bounded directional $p$-th moment assumptions, truncation
induces an unavoidable polynomial bias of order
$\mathcal{O}(\nu_{P,p}\eps^{1-1/p})$, where $\nu_{P,p}$ is the directional moment upper bound, yielding a sharp detectability floor below
which the null and alternative hypotheses are information-theoretically
indistinguishable, regardless of sample size.

Above this floor, we showed that simple second-order $U$-statistics achieve
near-optimal testing rates, preserving the classical $\Theta(\sqrt{d})$
high-dimensional sample complexity. This separation highlights a key distinction
between truncation and variance-driven noise: while truncation can systematically
distort the mean, second-order concentration remains stable under finite second
moments.

We further demonstrated that the moment-induced barrier is not absolute.
Under a directional median regularity assumption, geometric stability restores
linear $\mathcal{O}(\eps)$ robustness even for heavy-tailed distributions, revealing
an intermediate learning-hard regime before optimal testing rates are recovered.
These findings position truncation as a distinct and fundamentally challenging
corruption model, where detectability hinges on both tail decay and geometric
structure.

Natural extensions include testing and estimation under more general observation
mechanisms such as censoring or interval truncation, developing truncation-robust
covariance and regression procedures, and understanding computational–statistical
tradeoffs in the learning-hard regime identified by median regularity.

\bibliography{bibliography}

\newpage
\appendix

\section{Deferred Proof in Section~\ref{sec:prelims}}
\label{appendix:section3}
This appendix collects the proofs of the results in Section~\ref{sec:prelims}.
Lemma~\ref{lem:trunc-bias} establishes basic bounds on parameter distortion under truncation,
which are then specialized in Corollaries~\ref{cor:sg} and~\ref{cor:ht}.

\subsection{Sub-Gaussian and Finite Moment Heavy Tailed Distributions}

\paragraph{Light-tailed distributions}
If $X \sim \mathcal{N}(\mu_P, \Sigma_P)$, then for any unit vector $v$, 
$\langle v, X - \mu_P \rangle \sim \mathcal{N}(0, v^\top \Sigma_P v)$.
Hence, letting $Z \sim \mathcal{N}(0,1)$,
\[
\E|\langle v, X - \mu_P \rangle|^p 
= (v^\top \Sigma_P v)^{p/2}\, \E|Z|^p 
= (v^\top \Sigma_P v)^{p/2}\, 2^{p/2}\frac{\Gamma((p+1)/2)}{\sqrt{\pi}}.
\]
Using Stirling's approximation, $(\Gamma((p+1)/2))^{1/p} \asymp \sqrt{p}$ as $p \to \infty$ (see, e.g., \cite[App.~A.2]{vershynin2018hdp} or \cite[Sec.~7.1]{feller1971}).
Hence,
\[
\nu_{P,p} = (\E|Z|^p)^{1/p}\|\Sigma_P\|^{1/2} \asymp \sqrt{p}\,\|\Sigma_P\|^{1/2}.
\]
More generally (see also \cite[Theorem 2.5.2]{vershynin2018hdp}), if $X$ is sub-Gaussian with parameter $\sigma$, i.e., 
\[
(\E|\langle v, X - \mu_P\rangle|^p)^{1/p} \le C\sigma\sqrt{p}
\] 
for all $p \ge 1$ and $\|v\|=1$, 
then $\nu_{P,p} \le C\sigma\sqrt{p}$.

\paragraph{Heavy-tailed distributions.}
For heavy-tailed distributions, higher moments may grow faster than $\sqrt{p}$ or are even unbounded beyond some finite order. 
Intuitively, $\nu_{P,p}$ measures how large the $p$-th directional moment can be; 
it grows slowly (as $\sqrt{p}$) for light-tailed data but much faster when tails are heavy.

To make this precise, suppose that for some unit vector $v$, the projection $\langle v, X-\mu_P\rangle$ has a polynomially decaying tail:
\[
\Pr(|\langle v, X-\mu_P\rangle| > t) \asymp t^{-\kappa}, \qquad t\to\infty,
\]
for some tail exponent $\kappa>2$. Then 
\[\E|\langle v,X-\mu_P\rangle|^p<\infty
\]
only for $p<\kappa$, 
then $\nu_{P,p}$ remains bounded for $p<p_0=\kappa$, but diverges as $p\uparrow \kappa$.

For Gaussian or sub-Gaussian distributions, the truncation bias and covariance distortion both decay nearly linearly in~$\eps$.
\begin{corollary}[Light-tailed specialization: Gaussian and sub-Gaussian]
\label{cor:sg}
Suppose $P$ has sub-Gaussian one-dimensional marginals with parameter $\sigma>0$, that is
\[
(\E|\langle v,X-\mu_P\rangle|^p)^{1/p}\le C_0\,\sigma\sqrt{p}
\quad \text{for all } p\ge1,\ \|v\|=1.
\]
Then for any measurable $S\subset \R^d$ with $P(S) \ge 1 - \eps$ and $\eps \in (0, 1/2]$, 
\[
\| \mu_P - \mu_{P_S}\| \; \le\; C_1\, \sigma\,\eps\,\sqrt{\log(1/\eps)}, 
\]
for universal constants $C_1 >0$.
\smallskip

\noindent (Gaussian distribution.)
If $X\sim \mathcal{N}(\mu_P, \Sigma_P)$, then the same bounds hold with $\sigma^2=\|\Sigma_P\|$.
\end{corollary}

\begin{proof}
This follows from Lemma~\ref{lem:trunc-bias} by plugging in the sub-Gaussian moment growth
$\nu_{P,p}\le C_0\sigma\sqrt{p}$ and optimizing over $p$, which yields $p \asymp \log(1/\eps)$.
\end{proof}
Intuitively, sub-Gaussian moment control implies that, along any direction, the magnitude of typical fluctuations
is at scale at most $\sigma$, with large deviations becoming exponentially unlikely.
Therefore, even if an adversary removes an $\eps$-fraction of the probability mass by truncation,
the maximum resulting shift in the mean, which is governed by the largest values that can be removed at $\eps$-quantile scale, 
is nearly linear on $\sigma\,\eps$ (with an additional $\sqrt{\log(1/\eps)}$ factor) as the bulk of the mass is removed from points at distance $\Theta(\sigma)$ from the mean.

This behavior contrasts sharply with the heavy-tailed setting,
where removing an $\eps$-fraction of mass can excise regions
that carry a non-negligible share of the total variance.
The next corollary quantifies this degradation in robustness
as the tail index decreases.
\begin{corollary}[Heavy-tailed specialization: finite-$p_0$ moments]
\label{cor:ht}
Assume there exist $p_0\in[2,\infty)$ and $M>0$ such that
\[
\sup_{\|v\|=1}\Big(\E\,|\langle v,X-\mu_P\rangle|^{p_0}\Big)^{1/p_0}\ \le\ M.
\]
Then for any measurable $S\subset\R^d$ with $P(S)\ge 1-\eps$ and $\eps\in[0,1/2]$,
\begin{align*}
\|\mu_P-\mu_{P_S}\| &\le 2\,M\,\epsilon^{\,1-\frac{1}{p_0}}.\\
\end{align*}
\end{corollary}

\begin{proof}
By assumption, $\nu_{P,p_0}\le M$. Apply Lemma~\ref{lem:trunc-bias} with $p=p_0$.
\end{proof}

\paragraph{Examples.}
\begin{itemize}
\item \textbf{Student-$t$ tails.} 
If each one-dimensional projection $\langle v,X-\mu_P\rangle$ has a Student-$t_\kappa$ distribution (up to scale) with degrees of freedom $\kappa>2$, then moments exist for $p<\kappa$. Taking any $p_0\in[2,\kappa)$ and $M=\sup_{\|v\|=1}\big(\E|\langle v,X-\mu_P\rangle|^{p_0}\big)^{1/p_0}$,
\[
\|\mu_P-\mu_{P_S}\|\;\lesssim\; M\,\epsilon^{\,1-\frac{1}{p_0}}.
\]
As $\kappa$ decreases (heavier tails), the best admissible $p_0$ decreases and the exponents $1-\tfrac{1}{p_0}
$ correspondingly worsen.

\item \textbf{Pareto tails.} If $\Pr(|\langle v,X-\mu_P\rangle|>t)\asymp t^{-\alpha}$ for some $\alpha>2$ (uniformly over unit $v$ up to scale), then moments exist for $p<\alpha$. Choosing any $p_0\in[2,\alpha)$ and the corresponding $M$ yields the same bounds with exponents $1-\tfrac{1}{p_0}$. 
In particular, taking $p_0$ close to $\alpha$ gives the rate $\eps^{\,1-\frac{1}{\alpha}}$ for the mean deviation (up to constants).
\end{itemize}

\subsection{Truncation-Induced Parameter Distortion}

We first prove Lemma~\ref{lem:trunc-bias}, which quantifies the bias induced by truncation under finite moment assumptions.

\paragraph{Lemma 2 (Truncated vs. Non-truncated Parameters).}
Let \( P \) be a distribution and \( S \) a measurable set with \( P(S) \geq 1 - \eps \) for \( \eps \in [0, 1/2] \). Then for any \( p \geq 1 \):
\begin{enumerate}
    \item \textit{Mean Shift:} \( \|\mu_P - \mu_{P_S}\| \leq 2 \nu_{P,p} \eps^{1 - 1/p} \).
    \item \textit{Covariance Bound:} \( \|\Sigma_{P_S}\| \leq (1 - \eps)^{-2} \|\Sigma_P\| \leq 4 \|\Sigma_P\| \).
\end{enumerate}

\begin{proof}
\textbf{Mean difference bound:} Let $v\in\R^d$ be any unit vector, i.e. $\|v\|=1$. 
We start by expressing the directional difference of the means:
\[
|\langle \mu_{P}-\mu_{P_S},v\rangle| = \frac{1}{P(S)}|\EXP_{X\sim P}[\textbf{1}_{X\in S}\langle v,\mu_P-X\rangle]|.
\]
Since $P(S)\geq 1-\eps\geq 1/2$, we have $P(S)^{-1} \leq 2$, and hence
\[
\left|\langle \mu_{P}-\mu_{P_S},v\rangle \right| \leq 2|\EXP_{X\sim P}[\textbf{1}_{X\in S}\langle v,\mu_P-X\rangle]|.
\]
Next, note that
\[
\EXP_{X\sim P}[\langle v,\mu_P-X\rangle] = 0,
\]
so we can rewrite
\[
\EXP_{X\sim P}[\textbf{1}_{X\in S}\langle v,\mu_P-X\rangle] = \EXP_{X\sim P}[\langle v,\mu_P-X\rangle]-\EXP_{X\sim P}[\textbf{1}_{X\not\in S}\langle v,\mu_P-X\rangle]=-\EXP_{X\sim P}[\textbf{1}_{X\not\in S}\langle v,\mu_P-X\rangle],
\]
Applying H\"older's inequality with conjugate exponents $p$ and $q = \frac{p}{q-1}$ gives:
\[
\left|\EXP_{X\sim P}[\textbf{1}_{X\not \in S}\langle v,\mu_P-X\rangle] \right|
\leq (1-P(S))^{1-1/p}\,(\EXP_{X\sim P}|\langle v,X-\mu_P\rangle|^p)^{1/p}.
\]
Since $1 - P(S) \leq \eps$ and by the definition of $\nu_{P,p}$,
\[
\left|\EXP_{X\sim P}[\textbf{1}_{X\not \in S}\langle v,\mu_P-X\rangle] \right| \leq \eps^{1-1/p} \nu_{P,p}.
\]
Substituting this into the previous inequality yields,
\[
\left|\langle \mu_{P}-\mu_{P_S},v\rangle \right| \leq  2\eps^{1-1/p} \nu_{P,p},
\]
Finally, taking the supremum over all unit vectors $v$ gives
\[\|\mu_P-\mu_{P_S}\|\leq  2\eps^{1-1/p} \nu_{P,p}.
\] 

\noindent \textbf{Covariance difference bound:} To control the deviation between $\Sigma_P$ and $\Sigma_{P_S}$, let $X,Y\sim P$ be independent. Then
\[
\Sigma_{P}=\frac{1}{2}\EXP[(X-Y)^{\otimes 2}],
\]
where $a^{\otimes 2} \defeq aa^\top$. Since $(X-Y)^{\otimes 2}$ is positive semidefinite, 
\[
\Sigma_{P}\succeq \frac{1}{2}\EXP[(X-Y)^{\otimes 2}\textbf{ 1}_{X\in S}\textbf{1}_{Y\in S}] = P(S)^2\Sigma_{P_S},
\]
and thus $\|\Sigma_{P_S}\|\leq P(S)^{-2}\|\Sigma_P\|\leq (1-\varepsilon)^{-2}\|\Sigma_P\|\leq 4\|\Sigma_P\|$,
where the last inequality is due to the fact that $\varepsilon\leq \frac12$
\end{proof}

\section{Deferred Proof in Section~\ref{sec:moment_testing}}
\label{app:variance}
This appendix contains the technical proofs for the testing results in
Section~\ref{sec:moment_testing}. We first establish a variance bound for
the $U$-statistic used in our test, and then complete the proof of the
median-of-means testing procedure under truncation.

\subsection{Lemma 3 (Variance Bound for U-statistic).}
The following lemma bounds the variance of the $U$-statistic used to estimate
$\|\mu_{P_S}\|^2$, and is the main concentration ingredient in the proof of the
testing upper bound.
\paragraph{Lemma 3 (Variance Bound for U-statistic).}
Fix a distribution $Q$ and let $X_1,\dots,X_n \sim Q$ be i.i.d. Then the variance of the $U$-statistic
\[
F := \binom{n}{2}^{-1}\sum_{1\le i<j\le n}\langle X_i, X_j\rangle
\]
satisfies
\begin{equation}
\Var_{Q}(F)
\;\le\;
\frac{4}{n}\,\langle \mu_{Q}, \Sigma_{Q} \mu_{Q} \rangle
\;+\;
\frac{4}{n(n-1)}\,\tr(\Sigma_{Q}^2).
\end{equation}

\begin{proof}
   The reliability of $F$ as a test statistic depends on its concentration. We bound $\textrm{Var}(F)$ via Efron-Stein inequality:
\[
\Var(F)\leq \frac{1}{2}\sum_{i=1}^n\EXP[(F-F^{(i)})^2],
\]
where, for each $i\in[n]$, $F^{(i)} = F(X_1,\dots,X_{i-1},X_i',X_{i+1},\dots,X_n)$ replaces $X_i$ by an independent copy $X'_i\sim Q$. 

Conditioning on $X_i,X_i'$,
\[
F-F^{(i)} = \frac{1}{\binom{n}{2}}\sum_{j\neq i}\langle X_i-X_i',X_j\rangle
\]
Since the $\{X_j\}_{j\neq i}$ are i.i.d.\ 
with mean $\mu_Q$ and covariance $\Sigma_Q$,
\[
\mathbb{E}[F-F^{(i)}\mid X_i,X_i'] = \tfrac{2}{n}\langle X_i-X_i',\mu_Q\rangle,\qquad
\mathrm{Var}[F-F^{(i)}\mid X_i,X_i']
= \tfrac{4}{n^2(n-1)}(X_i-X_i')^\top\Sigma_Q(X_i-X_i').
\]  
It follows that
\begin{eqnarray*}\EXP[(F-F^{(i)})^2\mid X_i,X'_i] &=&  (\EXP[F-F^{(i)}\mid X_i,X'_i])^2 + \Var[F-F^{(i)}\mid X_i,X'_i] \\ &=&  \frac{4}{n^2}\langle X_i-X_i',\mu_P\rangle^2  + \frac{n-1}{\binom{n}{2}^2}\,\langle X_i-X_i',\Sigma_Q\,(X_i-X'_i)\rangle \\ &=& \frac{4}{n^2}\langle X_i-X_i',\mu_P\rangle^2 + \frac{4}{n^2(n-1)}{\tr}(\Sigma_Q\,(X_i-X'_i)^{\otimes 2}).\end{eqnarray*}
Taking expectations over $X_i,X_i'$ and summing over $i$ yields
\[
\mathrm{Var}(F)
\;\le\; \frac{4}{n}\,\langle \mu_Q,\Sigma_Q\mu_Q\rangle
+\frac{4}{n(n-1)}\,\mathrm{tr}(\Sigma_Q^2).
\;
\]
 
\end{proof}

\subsection{Median-of-means estimator}
In this subsection, we construct a simple test for mean detection under truncation
with constant error probability. 
We next show how to amplify the constant success probability in
Theorem~\ref{thm:mean-test-trunc-const} to an arbitrary confidence level by standard
probability boosting in Theorem~\ref{thm:achievable_rates}. This yields the final achievable testing rates under truncation.

\MeanTestTruncConstant*

\begin{proof}
The fact that the statistic $F$ is an unbiased estimator of the squared Euclidean norm of the truncated mean was noted in the main text: $\mathbb{E}_{P_S}[F]=\|\mu_{P_S}\|^2$.
Furthermore, by Lemma~\ref{lem:trunc-bias}, we have that $\|\mu_P-\mu_{P_S}\|\le \gamma$.

Combining the variance bound for $F$ in (\ref{eq:variance_bound_P}) with Chebyshev's inequality, with obtain that there exists a universal $c>1$ such that, for any $n\geq 2$:
\begin{equation}\label{eq:twosidedChebyshevF}\P\left[|F - \|\mu_{P_S}\|^2|\leq c\,\left(\|\mu_{P_S}\|\sqrt{\frac{\|\Sigma_P\|}{n}} + \frac{\sqrt{d}\|\Sigma_P\|}{n}\right)\right]\geq 2/3.\end{equation}

In the remainder of the proof, we set $\Delta:=\alpha/2 - \gamma>0$. By taking the constant $C>0$ in the assumptions of the theorem to be large enough, we may assume that: 
\begin{equation}\label{eq:assumen}n\geq \max\left\{2,\frac{\|\Sigma_P\|(c^2+4\sqrt{d})}{\Delta^2}\right\},\end{equation}
In the remainder of the proof, we will combine (\ref{eq:assumen}) with (\ref{eq:twosidedChebyshevF}), to prove that our test works as expected. We now distinguish the 2 cases:\\

\noindent (Soundness): Under the alternative hypothesis $\|\mu_P\|\ge\alpha$ we have
$\|\mu_{P_S}\|\ge (\alpha-\gamma)$ due to the triangle inequality. 

We {\em claim} that we have $F> \alpha^2/4$ whenever the event in (\ref{eq:twosidedChebyshevF}) holds (and $n$ is suitably large). Notice that this claim implies $\P(T=1) = \P(F>\alpha^2/4)\geq 2/3$, so that soundness ensues. 

To prove the claim, note that, whenever the event in (\ref{eq:twosidedChebyshevF}) holds, we have the following lower bound on $F$:
\begin{align}\nonumber F & \geq \|\mu_{P_S}\|^2 - c\,\left(\|\mu_{P_S}\|\sqrt{\frac{\|\Sigma_P\|}{n}} + \frac{\sqrt{d}\|\Sigma_P\|}{n}\right)\\ \label{eq:lowerboundF} & = \left(\|\mu_{P_S}\| - \frac{c}{2}\sqrt{\frac{\|\Sigma_P\|}{n}}\right)^2 - \frac{\|\Sigma_P\|(c^2 + 4c\sqrt{d})}{4n}.\end{align}
Using (\ref{eq:assumen}), we obtain\[\|\mu_{P_S}\| - \frac{c}{2}\sqrt{\frac{\|\Sigma_P\|}{n}}\geq \alpha - \gamma - \frac{\Delta}{2} = \frac{\alpha+\Delta}{2}>0\]
and
\[\frac{\|\Sigma_P\|(c^2 + 4c\sqrt{d})}{4n}\leq \frac{\Delta^2}{4}.\]
Plugging this back into (\ref{eq:lowerboundF}), we obtain that
\[F\geq \left(\frac{\alpha+\Delta}{2}\right)^2 - \frac{\Delta^2}{4}=\frac{\alpha^2}{4} + \frac{\alpha\Delta}{2}>\frac{\alpha^2}{4}\]
whenever the event in (\ref{eq:twosidedChebyshevF}) holds. This is precisely our claim, and finishes the proof of soundness. \\

\noindent (Completeness): Under the null hypothesis $\|\mu_P\|=0$ we have $\|\mu_{P_S}\|\leq \gamma$.

We {\em claim} that, in the present case, we have $F\leq \alpha^2/4$ whenever the event in (\ref{eq:twosidedChebyshevF}) holds (and $n$ is suitably large). Notice that this implies $\P(T=0) = \P(F\leq \alpha^2/4)\geq 2/3$. 

To prove the claim, we observe that, when the event (\ref{eq:twosidedChebyshevF}) holds, we can upper bound:
\begin{align}\nonumber F & \leq \|\mu_{P_S}\|^2 + c\,\left(\|\mu_{P_S}\|\sqrt{\frac{\|\Sigma_P\|}{n}} + \frac{\sqrt{d}\|\Sigma_P\|}{n}\right)\\ \label{eq:upperboundF} & =  \left(\|\mu_{P_S}\| + \frac{c}{2}\sqrt{\frac{\|\Sigma_P\|}{n}}\right)^2 + \frac{\|\Sigma_P\|( 4c\sqrt{d}-c^2)}{4n}.\end{align}
Using $\|\mu_{P_S}\|\leq \gamma$ and (\ref{eq:assumen}), we obtain:
\[\|\mu_{P_S}\| + \frac{c}{2}\sqrt{\frac{\|\Sigma_P\|}{n}}\leq \gamma + \frac{\Delta}{2}\mbox{ and }\frac{\|\Sigma_P\|(4c\sqrt{d}-c^2)}{4n}\leq \frac{\Delta^2}{4}.\]
Plugging this back into (\ref{eq:upperboundF}) gives
\[F\leq \left(\gamma + \frac{\Delta}{2}\right)^2 + \frac{\Delta^2}{4}\leq \left(\gamma+\Delta\right)^2 = \frac{\alpha^2}{4}\]
whenever the event in (\ref{eq:twosidedChebyshevF}) holds. This proves our claim and finishes the proof of completeness.\end{proof}

\TestingUnderTruncation*
\begin{proof}[Proof sketch] The result follows from splitting the samples into $C\log (1/\delta)$ blocks, performing the test from Theorem \ref{thm:mean-test-trunc-const} on each block, and then taking a majority vote. \end{proof}

\section{Proofs for Section~\ref{sec:median_regularity}}
\label{app:section5}
This appendix contains the proofs of the results in Section~\ref{sec:median_regularity}.
We first establish stability properties of directional medians under truncation,
then show how these properties enable robust mean estimation and improved testing
rates under structural regularity assumptions.
\subsection{Proof of Lemma~\ref{lem:median_stability}}

\paragraph{Lemma 6 (Directional Median Stability).}
Suppose that the pair $(P,S)$ is such that $P$ satisfies Assumption~\ref{assump:median-regularity} and 
$P(S)\ge 1-\varepsilon$ for some $\varepsilon<cr$.
Then for every $v\in\mathbb{S}^{d-1}$,
\[
\bigl|
\mathrm{median}(\langle X,v\rangle \mid X\in S)
- \langle \mu_P,v\rangle
\bigr|
\;\le\;
\frac{\varepsilon}{c}.
\]
\begin{proof}
Fix a direction $v\in\mathbb{S}^{d-1}$ and define
$Z_v = \langle X-\mu_P,v\rangle$.
By Assumption~\ref{assump:median-regularity}, $Z$ has median $0$ and density
$f_v(t)\ge c$ for all $|t|\le r$.

For any $t\in(0,r]$,
\[
\mathbb{P}(Z_v\ge t)
\;\le\;
\frac12 - \int_0^t f_v(s)\,ds
\;\le\;
\frac12 - ct .
\]
Conditioning on a truncation set $S$ with $P(S)\ge 1-\varepsilon$ removes at most
$\varepsilon$ mass, hence
\[
\mathbb{P}_S(Z_v\ge t)
\;\le\;
\mathbb{P}(Z_v\ge t)+\varepsilon
\;\le\;
\frac12 - ct + \varepsilon .
\]
In particular, if $t\in (\varepsilon/c,r]$ (a non-vacuous condition since $\varepsilon<cr$), $\mathbb{P}_S(Z_v\ge t)<1/2$. This implies that the median of $Z_v$ is $<t$ for all $t\in (\varepsilon/c,r]$, which implies that 
\[\textrm{median}(Z_v|X\in S) = \textrm{median}(\langle X,v\rangle |X\in S) - \langle \mu_P,v\rangle \leq \frac{\varepsilon}{c}.\]
A similar argument with $-Z_v = Z_{-v}$ shows that
\[\textrm{median}(Z_v|X\in S) = \textrm{median}(\langle X,v\rangle |X\in S) - \langle \mu_P,v\rangle \geq -\frac{\varepsilon}{c},\]
which proves the Lemma.  
\end{proof}

\subsection{Proof of Lemma \ref{lem:median_instability}}

\paragraph{Lemma 7 (Sharpness of median stability condition). }
Given $0<\xi<1/2$, $\varepsilon>2\xi$,   and $R>0$, there exists a distribution $P$ over $\R$ and a set $S\subset \R$ with the following properties.
\begin{enumerate}
\item $P$ satisfies Assumption \ref{assump:median-regularity} with constants $c,r>0$ satisfying $cr=\xi$;
\item $P(S)\geq 1-\varepsilon$;
\item \(
\bigl|
\mathrm{median}(\langle X,v\rangle \mid X\in S)
- \langle \mu_P,v\rangle
\bigr|
\;\ge\;
R.
\)
\end{enumerate}
Moreover, the truncated distribution $P_S$ satisfies the second part of Assumption \ref{assump:median-regularity} for certain constants $c_0,r_0>0$. 

\begin{proof}Recall that $0<\xi<\varepsilon<1/2$. For a non-degenerate closed interval $I\subset \R$, let  $L^{I}$  denote (unnormalized) Lebesgue measure over this interval. Now assume without loss of generality that $R>1$. Set $\eta:=1/2-\xi\in (0,1/2)$ and define:
\[P^R:= L^{[-R-\eta,-R]} + L^{[R,R+\eta]} + L^{[-\xi,\xi]}.\]
This is a compactly supported, absolutely continuous probability measure that is symmetric around $0$, therefore the mean and median of $P^R$ are both $0$. One can check that $P^R$ satisfies Assumption \ref{assump:median-regularity} with $c=1$, $r=\xi$, so that $cr=\xi$. 

Take $\alpha:=\min\{(\varepsilon - 2\xi),\eta\}>0$. Now set 
\[S_R:= [-R-\eta+\alpha,-R]\cup [1,+\infty).\]
Then $P^R(S^R)= 2\eta-\alpha= 1-\alpha -2\xi \geq 1-\varepsilon$. Moreover, the truncated measure $P^R_{S_R}$ satisfies:
\[P^R_{S_R}:= \frac{1}{1-\alpha-2\xi}\left(L^{[-R-\eta+\alpha,-R]} + L^{[R,R+\eta]}\right).\]
Since $L^{[R,R+\eta]}((-\infty,R])=0$ and $L^{[-R-\eta+\alpha,-R]}((-\infty,R]) = L^{[-R-\eta+\alpha,-R]}(\mathbb{R}) < L^{[R,R+\eta]}(\mathbb{R}),$ we see that:
\[P^R_{S_R}((-\infty,R]) = \frac{L^{[-R-\eta+\alpha,-R]}((-\infty,R])}{L^{[-R-\eta+\alpha,-R]}(\mathbb{R}) + L^{[R,R+\eta]}(\mathbb{R})}<1/2\]
and the median of $P^R(S_R)$ is $>R$. This finishes the proof. \end{proof}

\subsection{Proof of Theorem~\ref{thm:median_recovery}}

\MedianRecovery*
\begin{proof} Define $F:\R^d\to\R$ via the following recipe:
\[F(u):=\sup_{v\in \mathbb{S}^{d-1}}|\hat{m_v} - \langle u,v\rangle|\,\,\,(u\in \mathbb{R}^{d}).\]
Our estimator for $\mu_P$ is $\hat{\mu}\in \argmin_{u\in \mathbb{R}^d}F(u)$. Therefore,
\begin{align*}\|\hat{\mu} - \mu_P\| & = \sup_{v\in \mathbb{S}^{d-1}}|\langle v,\hat{\mu}-\mu_P\rangle|\\ & \leq \sup_{v\in \mathbb{S}^{d-1}}[|\langle v,\hat{\mu}\rangle -\hat{m}_v| + |\hat{m}_v-\langle v,\mu_P\rangle|]\\ &\leq F(\hat{\mu}) + F(\mu_P)\\ &\leq 2F(\mu_P).\end{align*}
Therefore, Theorem \ref{thm:median_recovery} follows if we can show that, the following inequality holds with probability $\geq 1-\delta$:
\begin{equation}\label{eq:allmedians}F(\mu_P) = \sup_{v\in \mathbb{S}^{d-1}}|\hat{m_v} - \langle \mu_P,v\rangle|\leq \frac{1}{c}\left(\varepsilon + C\sqrt{\frac{d+\log(1/\delta)}{n}}\right)=:t_*.\end{equation}
Here, $C>0$ is a universal constant chosen so that (\ref{eq:thm8whp1}) and (\ref{eq:thm8whp2}) below hold with suitably high probablity. 

To prove this, we first note that, by the same argument as in the proof of Lemma ~\ref{lem:median_stability}, the following holds for any $0<t\leq r$ and $v\in \mathbb{S}^{d-1}$:
\begin{align*}\P_S(\langle X,v\rangle \leq \langle \mu_P,v\rangle - t) & \leq \frac{1}{2}+\varepsilon-ct;\\  
\P_S(\langle X,v\rangle \geq \langle \mu_P,v\rangle + t) & \leq \frac{1}{2}+\varepsilon-ct.\end{align*}
 A standard VC concentration argument for halfspaces (VC dimension $\Theta(d)$) implies that, with a suitable universal choice of $C>0$, the following inequalities hold with probability $\geq 1-\delta$ simultaneously for all $0<t\leq c$ and $v\in \mathbb{S}^{d-1}$:
\begin{align}\label{eq:thm8whp1}\frac{1}{n}\sum_{i=1}^n\textbf{1}\{\langle X_i,v\rangle \leq \langle \mu_P,v\rangle - t\} & \leq \frac{1}{2}-ct+\varepsilon + C\sqrt{\frac{d+\log(1/\delta)}{n}};\\ \label{eq:thm8whp2}\frac{1}{n}\sum_{i=1}^n\textbf{1}\{\langle X_i,v\rangle \leq \langle \mu_P,v\rangle + t\} & \leq \frac{1}{2}-ct+\varepsilon + C\sqrt{\frac{d+\log(1/\delta)}{n}}.\end{align}
Recall that $t_*$ is the RHS of (\ref{eq:allmedians}). Our assumptions on $n$ and $\varepsilon$ (with a suitable choice of $C_0$) guarantee that $t_*<r$.
This means that we can take any $t\in (t_*,r]$ in inequalities (\ref{eq:thm8whp1}) and (\ref{eq:thm8whp2}), and any such $t$ makes the RHS of these inequalities strictly smaller than $1/2$. Recalling that $\hat{m}_v$ is the sample median of $\langle X_i,v\rangle$, we obtain that the following inequalities hold with probability at least $1-\delta$, simultaneously for all $v\in \mathbb{S}^{d-1}$ and $t\in (t_*,r]$:
\[|\hat m_v -  \langle \mu_P,v\rangle|<t.\]
Letting $t\searrow t_*$, we obtain that  (\ref{eq:allmedians}) holds with probability $\geq 1-\delta$, as desired.\end{proof}

\subsection{Proof of Theorem~\ref{thm:testing_regularity}}

\TestingRegularity*
\begin{proof}
The second statement in the Theorem follows from \cite[Lemma 3.5]{canonnegaussian}, so we focus on the first statement.
Let $X_1,\dots,X_n\sim P_S$ be i.i.d. samples, and let $\hat{\mu}$ be the estimator from Theorem \ref{thm:median_recovery}. Our test takes the form:
\[T:=\textbf{1}\left\{\|\hat{\mu}\|>\alpha/2\right\}.\]
A sufficient condition for $T$ to give the right answer, under $H_0$ and $H_1$ both, is that
\[\|\hat{\mu}-\mu_P\|<\frac{\alpha}{2}.\]
Combining our assumptions on $P$, $\varepsilon$ and $n$ with Theorem \ref{thm:median_recovery}, we obtain that $\P\{\|\hat{\mu}-\mu_P\|<\alpha/2\}\geq 1-\delta$. This finishes the proof.\end{proof}

\end{document}